\documentclass{article}

\usepackage[accepted]{eiml_icml2026}
\usepackage[utf8]{inputenc}
\usepackage{microtype}
\usepackage{graphicx}
\usepackage{hyperref}
\usepackage{subcaption}
\usepackage{placeins}
\usepackage{amsfonts}  % For math blackboard fonts% Attempt to make hyperref and algorithmic work together better:
\usepackage{amsmath,amssymb,amsthm}
\usepackage{mathtools}
\newtheorem{theorem}{Theorem}
\usepackage{stfloats} % Add this to your preamble
\usepackage{multirow}
\usepackage{booktabs}
\newtheorem{proposition}{Proposition}
\DeclareMathOperator{\Var}{Var}

\DeclareMathOperator{\E}{E}

\icmltitlerunning{Exposure Bias as Epistemic Uncertainty in Recursive Forecasting}

\begin{document}

\twocolumn[
% \icmltitle{Exposure Bias as Epistemic Uncertainty in Recursive Forecasting}
\icmltitle{Exposure Bias as Epistemic Underidentification in Recursive Forecasting}

% List of authors (only shown if [accepted] option is used)
\begin{icmlauthorlist}
\icmlauthor{Riku Green}{univ}
\icmlauthor{Zahraa S. Abdallah}{univ}
\icmlauthor{Telmo M Silva Filho}{univ}

\end{icmlauthorlist}

\icmlaffiliation{univ}{University of Bristol, UK}
% \icmlaffiliation{equal}{Equal contribution}

\icmlcorrespondingauthor{Riku Green}{riku.green@bristol.ac.uk}

% Keywords for PDF metadata
\icmlkeywords{Machine Learning, Epistemic Intelligence, ICML}

\vskip 0.3in
]

% This prints the affiliations and the "Under Review" notice
\printAffiliationsAndNotice{}

\begin{abstract}
Recursive multi-step forecasting is usually framed as distribution shift: models are trained on observed histories but deployed on their own predictions. We show this framing is incomplete by proving that, under partial observability or state truncation, recursive rollout is also an epistemic underidentification problem. Even with deterministic latent dynamics, one-step Bayes supervision identifies behavior only on observed contexts and need not identify the deployed recursive predictor once rollout queries self-generated induced states whose correct local targets are not determined by numeric state alone. We formalize this with induced states $Z$ and provenance variables $P$, and derive a decomposition of induced-state error into teacher-forcing/rollout mismatch, representation--class approximation, and provenance information gaps. 
Empirically, we show that rollout enters a distinct induced-state regime, that fixed induced states define a distinct local corrective task, and that closed-loop gains arise not only from local adaptation but also from changing the induced states visited during rollout. Using a simple binary provenance encoding, provenance-aware correction can further improve performance, though gains are conditional rather than uniform.
These results recast exposure bias as reasoning under self-induced epistemic uncertainty.
\end{abstract}

\section{Introduction}

Autoregressive sequence prediction underpins applications from language generation \cite{radford2019language_gpt2} to dynamical forecasting \cite{green2025stratify}. 
But models are trained on observed histories and later act on states generated by their own predictions which can induce compounding error through exposure bias \cite{huszar2015not}. 
In this paper, we study this problem in recursive multi-step forecasting with numerical targets. Recursive rollout is simple and scalable, but early errors perturb later inputs and can compound across the horizon \cite{taieb2015bias}.

Exposure bias is usually framed as this train--test mismatch: teacher forcing supervises the model on observed contexts, while deployment proceeds on self-generated ones \cite{bengio2015scheduled}. This view motivates DAgger-style aggregation, Professor Forcing, and forecasting-specific mixed-regime training \cite{ross2011reduction,lamb2016professor,sangiorgio2020robustness}. But it leaves a more basic question open: once rollout begins, what prediction problem is the model actually solving?

We argue that under partial observability, noise, or state truncation, recursive forecasting is not only a distribution-shift problem. It can also be one of \emph{epistemic underidentification}. One-step supervision constrains behavior only on observed contexts, while rollout queries the model on induced states whose correct local targets may not be determined by the represented state alone. Recursive failure then reflects not only unfamiliar inputs, but also missing information in the state given to the predictor. This links exposure bias to epistemic uncertainty from representation insufficiency rather than irreducible noise \cite{hullermeier2021aleatoric,kendall2017uncertainties}.

This view also changes how correction should be interpreted. Once training includes rollout-induced states, the learner need not face one homogeneous prediction problem. The local target may depend not only on the induced state, but also on its \emph{provenance}, such as rollout depth or which coordinates are observed rather than model-generated. This is consistent with predictive-state and information-state views of partially observed dynamical systems, where state summaries must preserve the information needed for future prediction and control \cite{littman2001predictive,singh2012predictive,subramanian2022approximate}.

Our goal is not to propose a new correction method, but to clarify the mechanism behind recursive forecasting failure. 
We ask when: rollout creates a distinct induced-state regime, prediction on fixed induced states differs from the original observed-state problem, and when correction helps by changing the states visited in closed loop.
We also test whether a simple provenance mask helps in these fixed-state and closed-loop settings. 
Recursive forecasting thus provides a concrete testbed for epistemic uncertainty in sequential learning systems. Although our theory and experiments focus on numerical forecasting, the underlying concern is broader: autoregressive systems often act on self-generated histories that may be under-specified by their state representation.
Our contributions are:

\begin{itemize}
    \item We show that exposure bias is not only distribution shift: under partial observability or state truncation, the one-step Bayes objective need not identify recursive rollout itself, even among recursive predictors.

    \item We formalize recursive forecasting using induced states and provenance, show how provenance can resolve clashes between observed and induced corrective targets, and derive a decomposition into teacher-forcing/rollout mismatch, representation--class approximation, and provenance information gaps.

    \item We provide empirical evidence that rollout enters a distinct induced-state regime and that fixed induced states define a distinct local corrective task, with heterogeneous gains from frozen-state relearning across datasets and horizons.

    \item We show that closed-loop correction can help not only by improving local prediction on fixed induced states, but also by changing the induced-state regime encountered during rollout; under the present provenance encoding, such gains are conditional rather than uniform.
\end{itemize}

\section{Exposure Bias as Epistemic Underidentification}
\label{sec:theory}

Exposure bias in sequential prediction is usually framed as covariate shift:
models are trained under teacher forcing on observed prefixes but deployed on
their own predictions, so early errors perturb future inputs and compound over
rollout \cite{williams1989learning,bengio2015scheduled}.
This has motivated training on learner-induced states through scheduled
sampling \cite{bengio2015scheduled}, DAgger-style aggregation
\cite{ross2011reduction}, adversarial alignment of train and rollout dynamics
\cite{lamb2016professor}, sequence-level objectives
\cite{ranzato2015sequence}, and forecasting-specific free-running or mixed-regime
curricula \cite{venkatraman2015improving,sangiorgio2020robustness,
sangiorgio2021forecasting,teutsch2022flipped,vlachas2024learning}. Classical
analyzes also show that one-step optimal recursive forecasting need not
coincide with multi-step optimal forecasting, often by contrasting recursive
and direct strategies or by exploiting noise and horizon-specific target
mismatch \cite{taieb2015bias,green2025stratify, green2024time,yoon2022robust}. Our claim is
earlier and different: under partial observability or state truncation, the
one-step Bayes objective may fail to \emph{identify recursive rollout even
among recursive predictors themselves}. Thus the issue is not only mismatch
between one-step and multi-step objectives, but underidentification of the
deployed recursive map induced by an insufficient state representation.

This perspective is closer to epistemic uncertainty from missing knowledge or
insufficient representation than to irreducible noise
\cite{hullermeier2021aleatoric,kendall2017uncertainties}. It is also aligned
with predictive-state and information-state views of partially observed
dynamical systems, where a valid state summary must preserve the information
needed for future prediction and control
\cite{littman2001predictive,singh2012predictive,subramanian2022approximate}.
Recent work has studied multi-step forecasting through epistemic
bias--variance-style decompositions and Jacobian amplification of predictor
uncertainty \cite{green2025epistemic}, but our question is different: what
prediction problem does recursive rollout itself create? We argue that
rollout can produce induced states whose correct local target is not
identified by numeric state alone, so corrective methods such as scheduled
sampling \cite{bengio2015scheduled} may encounter \emph{target clashes} unless
the state is augmented with information about \emph{how} it was formed. We
call this additional information \emph{provenance}. Proofs and expanded
derivations are deferred to Appendix~\ref{app:ext_theory}.

\begin{figure*}[t]
    \centering
    \begin{subfigure}[t]{0.42\textwidth}
        \centering
        \includegraphics[width=\linewidth]{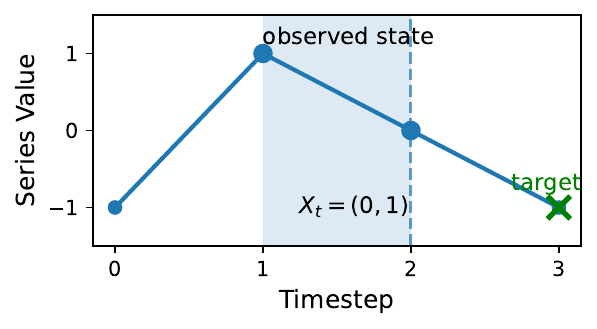}
        \caption{Observed state $(0,1)$ with target $-1$.}
    \end{subfigure}\hfill
    \begin{subfigure}[t]{0.42\textwidth}
        \centering
        \includegraphics[width=\linewidth]{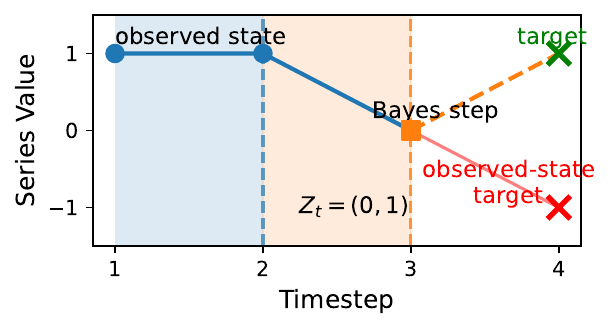}
        \caption{Bayes-induced rollout state $(0,1)$ with local target $+1$.}
    \end{subfigure}
    \caption{\textbf{Provenance can resolve a clash between observed and
    corrective targets.} The same numeric state $(0,1)$ occurs both on the
    observed support and as a self-generated rollout state under the
    Bayes-optimal one-step predictor. These two occurrences require different
    next-step targets. A predictor that uses only the numeric state must
    conflate them; augmenting the input with a binary observed/generated
    provenance tag separates the observed regime from the induced regime and
    resolves the clash. Full details are given in
    Appendix~\ref{app:toy-provenance}.}
    \label{fig:toy_prov}
\end{figure*}

\paragraph{Setup.}
Let $X_t$ denote the observed context, let $Y_{t+1}$ denote the next target,
and write
\[
  M := \operatorname{supp}\!\bigl(\mathcal{L}(X_t)\bigr)
\]
for the support of the observed-context distribution. For a measurable
predictor $g$, define the closed-loop update
\[
  T_g(x_1,\ldots,x_{\hat p})
  :=
  \bigl(g(x_1,\ldots,x_{\hat p}),x_1,\ldots,x_{\hat p-1}\bigr),
\]
and the two-step recursive forecast
\[
  \Phi_g(x) := g\!\bigl(T_g(x)\bigr).
\]
Under partial observability or state truncation, the represented state $X_t$
need not determine the next target even when the latent dynamics are
deterministic. We therefore ask whether the one-step Bayes objective identifies
the recursive predictor used at deployment.

\begin{theorem}[One-step Bayes optimality underidentifies recursive rollout]
\label{thm:one-step-optimality-recursive-rollout}
Let $g^\star$ be Bayes-optimal for one-step prediction under squared loss, and
suppose there exists $x\in M$ such that
\[
  T_{g^\star}(x)\notin M.
\]
Then the one-step Bayes objective identifies the predictor only on $M$, not on
rollout-induced states outside $M$. Consequently, there exist measurable
predictors $g_1$ and $g_2$ such that
\[
  g_1(X_t)=g_2(X_t)=g^\star(X_t)
  \qquad \text{a.s.},
\]
but
\[
  \Phi_{g_1}(x)\neq \Phi_{g_2}(x).
\]
Hence two recursive predictors can attain identical one-step Bayes risk while
inducing different recursive multi-step forecasts.
\end{theorem}

Theorem~\ref{thm:one-step-optimality-recursive-rollout} should not be read as
the well-known observation that one-step optimality need not imply multi-step
optimality from \cite{taieb2015bias}. The sharper point is that, among recursive predictors with identical one-step Bayes risk, recursive
rollout need not be identified once deployment queries self-generated states
outside the observed support. A first-order two-step expansion makes this
distinction explicit. If
$\delta(x):=g_1(x)-g_2(x)$ and $g_1,g_2$ are $C^1$ near $x$ and $T_{g_2}(x)$,
then Appendix~\ref{app:two-step-divergence} shows
\begin{align}
  \Phi_{g_1}(x)-\Phi_{g_2}(x)
  &=
  \bigl(g_1-g_2\bigr)\!\bigl(T_{g_2}(x)\bigr)
  \notag\\
  &\quad+
  \partial_1 g_1\!\bigl(T_{g_2}(x)\bigr)\,\delta(x)
  +
  o\!\bigl(|\delta(x)|\bigr).
  \label{eq:two-step-divergence}
\end{align}
The first term is disagreement at the induced state queried by rollout; this is
the underidentification mechanism of
Theorem~\ref{thm:one-step-optimality-recursive-rollout}. The second is a Jacobian of the recursive composition map, similar to recent work but on induced states rather than variance bias decompositions \cite{green2025epistemic}. Thus our novel angle is that recursive forecasting failure can arise not only
from compounding dynamics, but also from underidentified local behavior on
self-generated states.

\paragraph{Induced states, corrective targets, and provenance.}
Fix a rollout depth $h\ge 1$, and let
\[
  Z_h=\psi_h(X_t), \qquad P_h=\pi_h(X_t)
\]
be measurable deterministic functions of $X_t$, where $Z_h$ is the induced
numeric state and $P_h$ is provenance information describing how that state was
formed. For the local corrective target $Y_{t+h+1}$, define the Bayes risks
\begin{align*}
  R_h^\star(Z_h)
  &:=
  \inf_q \E\!\big[(Y_{t+h+1}-q(Z_h))^2\big], \\
  R_h^{\mathrm{prov},\star}(Z_h,P_h)
  &:=
  \inf_r \E\!\big[(Y_{t+h+1}-r(Z_h,P_h))^2\big].
\end{align*}

\begin{theorem}[Provenance can strictly reduce Bayes risk on induced states]
\label{thm:information-sandwich-induced-states-provenance}
Under squared loss,
\begin{align}
\label{eq:prov-gap}
  R_h^\star(Z_h)-R_h^{\mathrm{prov},\star}(Z_h,P_h)
  &= \\
  \E\!\left[
    \Var\!\Big(
      \E[Y_{t+h+1}\mid Z_h,P_h]
      \,\Big|\,
      Z_h
    \Big)
  \right]
  \notag
  &\ge 0.
\end{align}
Hence conditioning on provenance cannot increase Bayes risk, and the inequality
is strict exactly when
\[
  \E[Y_{t+h+1}\mid Z_h,P_h]
  \neq
  \E[Y_{t+h+1}\mid Z_h]
\]
on a set of positive probability. Provenance cannot create information beyond
what was already present in the original observed context, but it can recover
information lost in the map $X_t\mapsto Z_h$. A full derivation is given in
Appendix~\ref{app:ext_theory}.
\end{theorem}

Figure~\ref{fig:toy_prov} gives a minimal illustration of
Eq.~\eqref{eq:prov-gap}; full details are in
Appendix~\ref{app:toy-provenance}. In the toy delay system, the same numeric
state $(0,1)$ appears both as an \emph{observed} state with target $-1$ and as
a \emph{Bayes-induced rollout state} with local corrective target $+1$. Thus
corrective training on induced states can clash with supervision from the
observed dataset at the same numeric input, a failure mode that standard
rollout-mixing methods do not distinguish \cite{bengio2015scheduled}. A binary
observed/generated provenance tag resolves this clash by separating the two
regimes.

\begin{figure}[t]
    \centering
    \includegraphics[width=\linewidth]{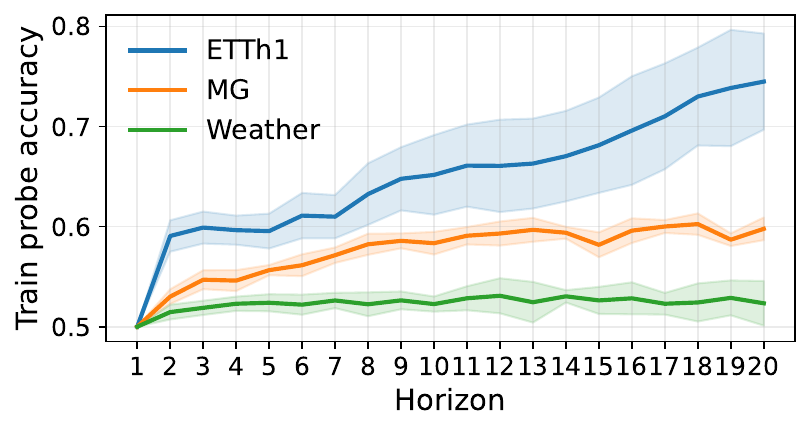}
    \caption{
\textbf{Rollout regime check.}
Linear-probe accuracy for distinguishing observed contexts $X_t$ from teacher-forced induced states $Z_h=\psi_h^{g_{\mathrm{TF}}}(X_t)$ across rollout depth $h$ (mean $\pm$ standard deviation). Accuracy generally increases with depth, most strongly on ETTh1, moderately on MG, and weakly on Weather, suggesting that rollout progressively leaves the observed-state regime.
}
    \label{fig:regime_check}
\end{figure}

\begin{figure*}[t]
    \centering

    \begin{subfigure}[t]{0.32\textwidth}
        \centering
        \includegraphics[width=\linewidth]{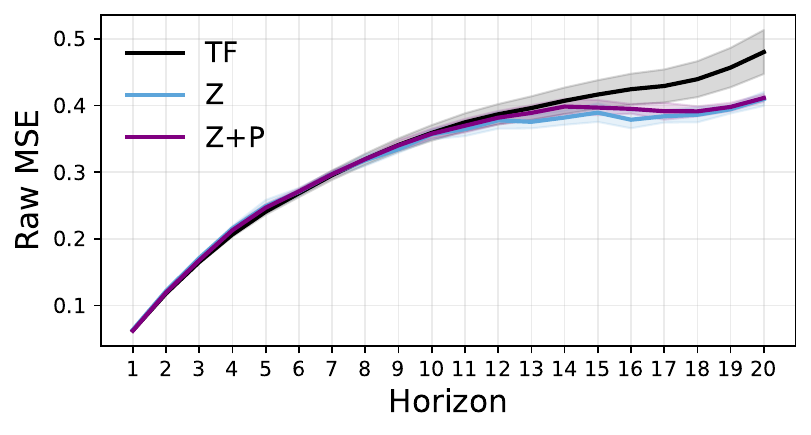}
        \caption{MG}
    \end{subfigure}\hfill
    \begin{subfigure}[t]{0.32\textwidth}
        \centering
        \includegraphics[width=\linewidth]{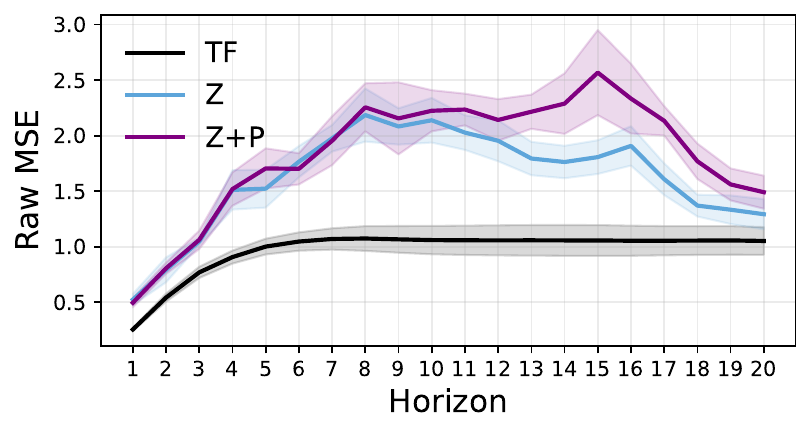}
        \caption{ETTh1}
    \end{subfigure}\hfill
    \begin{subfigure}[t]{0.32\textwidth}
        \centering
        \includegraphics[width=\linewidth]{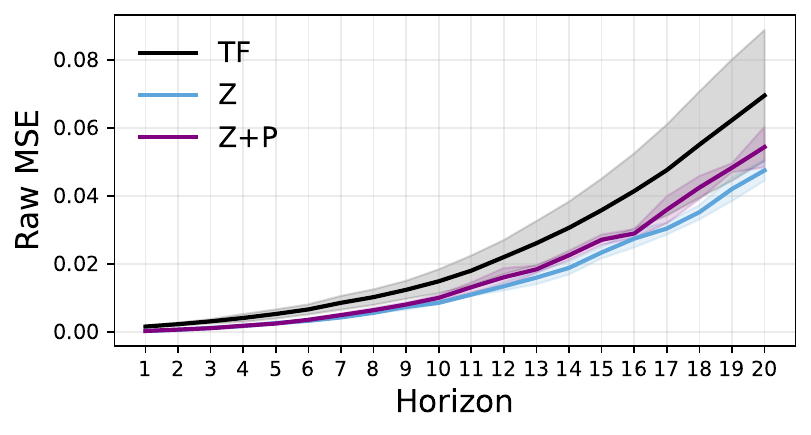}
        \caption{Weather}
    \end{subfigure}

    \caption{
\textbf{Frozen induced-state evaluation.}
For each horizon $h$, we freeze teacher-forced induced states $Z_h=\psi_h^{g_{\mathrm{TF}}}(X_t)$ and compare the original teacher-forced predictor (TF), a probe on $Z_h$ alone (Z), and a provenance-aware probe on $(Z_h,P_h)$ (Z+P). Relearning on fixed induced states is heterogeneous: Z and Z+P often match or outperform TF on MG and Weather, but underperform strongly on ETTh1. Z+P remains close to Z, suggesting limited extra gains from the provenance encoding.
}
    \label{fig:induced_decomposition}
\end{figure*}

\paragraph{Teacher-forced rollout as a distinct prediction object.}
Let $g_{\mathrm{TF}}$ denote the predictor learned by one-step empirical risk
minimization on observed pairs $(X_t,Y_{t+1})$, and define the depth-$h$
teacher-forced induced state
\[
  Z_h=\psi_h^{g_{\mathrm{TF}}}(X_t).
\]
Recursive deployment then presents not the direct horizon-$h$ task
$X_t\mapsto Y_{t+h}$, but the local next-step correction problem
\[
  (Z_h,P_h)\mapsto Y_{t+h+1}.
\]
% Furthermore, this local optimality can directly conflict with rollout MSE under uncertainty \cite{green2026expectations}. 
Furthermore, under conditional uncertainty, this local target optimality can be structurally incompatible with realistic multi-step rollouts~\cite{green2026expectations}.
This raises three separate questions: whether the teacher-forced predictor is
already adapted to induced states, whether the induced representation is well
matched to the chosen function class, and whether provenance adds
target-relevant information.

Let
\[
  R_h(q):=\E\!\big[(Y_{t+h+1}-q(Z_h))^2\big]
\]
denote local next-step risk on induced states, and let $\mathcal Q$ be a
predictor class on $Z_h$. Then Appendix~\ref{app:ext_theory} shows
{\small
\begin{multline}
\label{eq:decomp}
  R_h^Z(g_{\mathrm{TF}})
  -
  R_h^{\mathrm{prov},\star}(Z_h,P_h)
  =
  \underbrace{
    \Big(
      R_h^Z(g_{\mathrm{TF}})
      -
      \inf_{q\in\mathcal Q} R_h^Z(q)
    \Big)
  }_{\text{teacher-forcing/rollout mismatch}} \\
  +
  \underbrace{
    \Big(
      \inf_{q\in\mathcal Q} R_h^Z(q)
      -
      R_h^\star(Z_h)
    \Big)
  }_{\text{representation--class gap}}
  +
  \underbrace{
    \Big(
      R_h^\star(Z_h)
      -
      R_h^{\mathrm{prov},\star}(Z_h,P_h)
    \Big)
  }_{\text{provenance gap}}.
\end{multline}
}
All three terms are nonnegative. The first measures how poorly the
teacher-forced predictor transfers to the induced-state task, the second is a
function-class approximation gap on the induced representation, and the third
is the recoverable information loss from omitting provenance.

Taken together, Theorem~\ref{thm:one-step-optimality-recursive-rollout},
Eq.~\eqref{eq:two-step-divergence}, and
Theorem~\ref{thm:information-sandwich-induced-states-provenance} recast
exposure bias as more than train--test mismatch. Under insufficient
representation, recursive deployment is a problem of prediction on
self-generated states whose local meaning may be underidentified by numeric
state alone.

\section{Empirical Findings}
\label{sec:empirical_findings}

The theory makes four empirical predictions. First, if the identification
failure of Theorem~\ref{thm:one-step-optimality-recursive-rollout} is relevant
in practice, rollout should enter a state regime increasingly distinct from the
observed training distribution. Second, fixed induced states should define a
distinct local corrective task rather than merely the original one-step task on
shifted inputs. Third, if provenance carries target-relevant information beyond
numeric induced state, it should sometimes improve local or deployable
correction, though not necessarily under every encoding. Fourth, if recursive
correction helps by changing future states as well as by improving local
prediction, frozen-state and closed-loop evaluations should diverge.

We test these predictions on three datasets. Details of the implementation are
given in Appendix~\ref{app:experimental_setup}. For clarity, the main text
reports MLP results; Appendix~\ref{app:gru_robustness} shows that GRU
experiments exhibit the same qualitative patterns.

\subsection{Rollout reaches a distinct induced-state regime}
\label{subsec:rollout_regime}

Theorem~\ref{thm:one-step-optimality-recursive-rollout} becomes practically
relevant only if recursive deployment queries states unlike those seen during
one-step training. To test this, we train a linear probe to distinguish
observed contexts $X_t$ from teacher-forced induced states $Z_h$ at rollout
depth $h$.

Probe accuracy rises with depth on all datasets, with the strongest separation
on ETTh1, moderate separation on MG, and only weak separation on Weather
(Figure~\ref{fig:regime_check}). Thus rollout progressively moves the predictor
into a regime distinct from the observed training contexts. This is a
diagnostic rather than a proof of underidentification, but it establishes the
practical precondition for the theoretical problem: recursive deployment is not
confined to the one-step training regime.

\subsection{Fixed induced states define a distinct local corrective task}
\label{subsec:frozen_induced_states}

Equation~\eqref{eq:decomp} suggests separating two aspects of recursive
correction: adaptation to a fixed induced-state task, and changes in the
induced states visited during rollout. We first isolate the former. For each
horizon $h$, we freeze teacher-forced induced states $Z_h$ and compare the
original teacher-forced predictor $g_{\mathrm{TF}}$ with probes trained on
$Z_h$ alone and on $(Z_h,P_h)$ for predicting $Y_{t+h+1}$
(Figure~\ref{fig:induced_decomposition}). Because the induced-state
distribution is held fixed, differences here reflect only adaptation to the
local corrective task.

The resulting pattern is strongly dataset-dependent. On MG and Weather, probes
trained directly on induced states often match or outperform TF, indicating
that part of the teacher-forcing/rollout gap is recoverable by relearning on
the frozen task. On ETTh1, by contrast, both induced-state probes substantially
underperform TF. This is important for the theory: if exposure bias were only a
matter of fitting a better local regressor on self-generated inputs, one would
expect frozen-state relearning to help much more uniformly. Instead, the
results show that induced states can define a genuinely different local
prediction problem whose difficulty depends on the induced representation, the
target, the estimator class, and the dataset.

Across datasets, the provenance-aware probe remains close to the $Z_h$-only
probe. This does not contradict
Theorem~\ref{thm:information-sandwich-induced-states-provenance}, which is a
Bayes-level statement. Rather, it shows that under the present binary
provenance encoding and restricted probe class, only a limited portion of any
theoretical provenance gap is recovered in the frozen setting. Thus the frozen
results support two claims at once: induced-state correction is not the same
task as one-step prediction, and the usefulness of provenance is conditional on
whether the chosen encoding exposes target-relevant structure.

\begin{table}[t]
\centering
\setlength{\tabcolsep}{4pt}
\caption{
\textbf{Deployable rollout performance relative to teacher forcing.}
Rollout MSE of scheduled sampling (SS) and provenance-aware scheduled sampling (SSP), each normalized by the teacher-forced (TF) rollout MSE aggregated over datasets into uniformly split horizon buckets.
Values below $1$ indicate an improvement over TF.
}
\label{tab:rollout_bucket_rel}
\begin{tabular}{llcc}
\toprule
Bucket & Dataset & SS/TF & SSP/TF \\
\midrule
\multirow{3}{*}{Early}
  & ETTh1   & $1.040 \pm 0.006$ & $\mathbf{0.861 \pm 0.080}$ \\
  & MG      & $\mathbf{0.972 \pm 0.006}$ & $1.036 \pm 0.017$ \\
  & Weather & $1.002 \pm 0.062$ & $1.580 \pm 0.541$ \\
\midrule
\multirow{3}{*}{Mid}
  & ETTh1   & $1.071 \pm 0.011$ & $\mathbf{0.905 \pm 0.122}$ \\
  & MG      & $\mathbf{0.925 \pm 0.016}$ & $\mathbf{0.979 \pm 0.032}$ \\
  & Weather & $\mathbf{0.990 \pm 0.091}$ & $\mathbf{0.979 \pm 0.327}$ \\
\midrule
\multirow{3}{*}{Late}
  & ETTh1   & $1.059 \pm 0.013$ & $\mathbf{0.957 \pm 0.130}$ \\
  & MG      & $\mathbf{0.887 \pm 0.025}$ & $\mathbf{0.870 \pm 0.049}$ \\
  & Weather & $\mathbf{0.981 \pm 0.110}$ & $\mathbf{0.999 \pm 0.312}$ \\
\bottomrule
\end{tabular}
\end{table}

\subsection{Closed-loop correction helps partly by changing visited states}
\label{subsec:ss_bridge}

Frozen-state evaluation isolates local adaptation, but recursive deployment also
depends on which induced states are visited later in rollout. If correction
helped only by improving a fixed local prediction task, frozen-state gains
would track deployable rollout gains. When they do not, part of the benefit
must come from changing the induced-state regime itself.

Table~\ref{tab:rollout_bucket_rel} shows that closed-loop correction can
improve deployable recursive forecasting. The clearest pattern is on ETTh1,
where SSP improves over TF across all horizon buckets and SS does not. On MG,
both SS and SSP improve over TF in the mid and late buckets, while Weather is
more mixed and SSP has higher variance. These results are consistent with the
theoretical picture that recursive correction can matter not only through local
re-fitting on induced states, but also through the states it causes the model
to encounter thereafter.

Figure~\ref{fig:ss_bridge_cross_state} helps explain this bridge result.
Unlike Table~\ref{tab:rollout_bucket_rel}, which reports deployable rollout
error, Figure~\ref{fig:ss_bridge_cross_state} fixes the induced-state source
and measures raw next-step MSE across rollout depth. At deeper depths,
SSP-induced states often yield lower next-step error not only for SSP itself,
but also for TF and SS. Thus the gain is not only that SSP is a better local
predictor on a fixed task; SSP also tends to generate states that are locally
easier for multiple predictors. This is the empirical signature of the theory:
recursive correction changes not only the predictor applied to induced states,
but the induced-state regime itself.

Taken together, the mismatch between frozen-state and closed-loop results
supports the view that recursive deployment is not simply a fixed supervised
problem on shifted inputs. It is a feedback system in which correction can
improve both local prediction and the future induced states.

\paragraph{Provenance is suggestive, but not isolated.}
\label{subsec:prov_interpretation}

The toy example in Figure~\ref{fig:toy_prov} shows that provenance can in
principle resolve clashes between observed and induced corrective targets at
the same numeric state. Our real-data experiments do not isolate that clash
mechanism directly, but they are consistent with it. In particular, SSP
sometimes improves where provenance-free correction does not fully account for
the gain, most clearly on ETTh1. At the same time, the mixed frozen-state Z+P
results and the higher variance of SSP show that the present binary provenance
encoding is weak and not uniformly effective. We therefore interpret our simple provenance mask
empirically as a conditional source of useful information, not as a universal
remedy.

\begin{figure*}[t]
    \centering

    \begin{subfigure}[t]{0.32\textwidth}
        \centering
        \includegraphics[width=\linewidth]{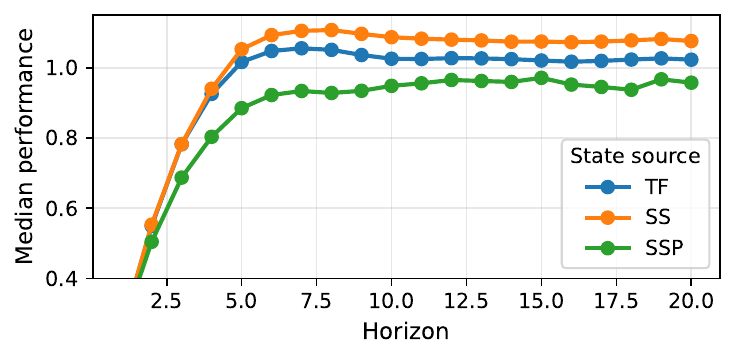}
        \caption{TF Model}
    \end{subfigure}\hfill
    \begin{subfigure}[t]{0.32\textwidth}
        \centering
        \includegraphics[width=\linewidth]{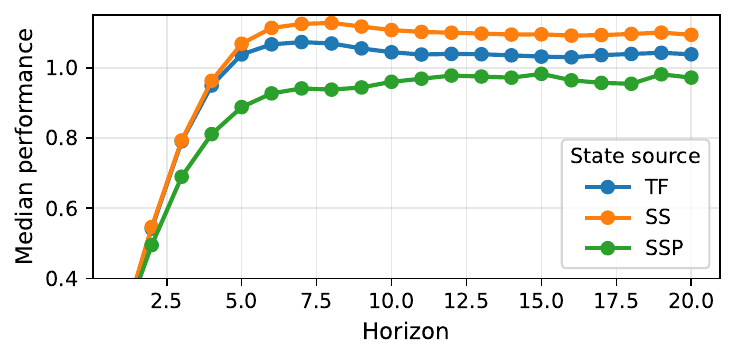}
        \caption{SS Model}
    \end{subfigure}\hfill
    \begin{subfigure}[t]{0.32\textwidth}
        \centering
        \includegraphics[width=\linewidth]{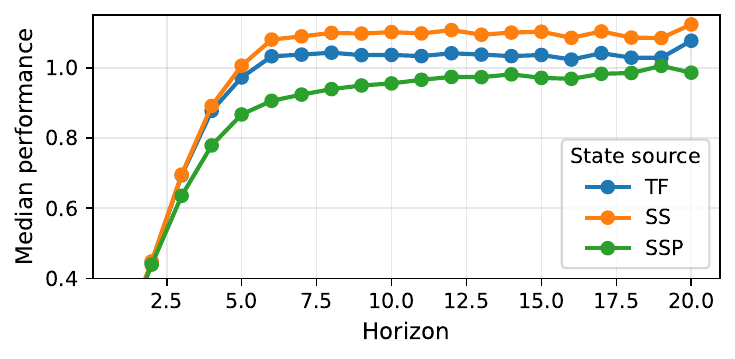}
        \caption{SSP Model}
    \end{subfigure}

\caption{
\textbf{ETTh1 raw cross-state next-step MSE.}
For each evaluated predictor (TF, SS, SSP), we measure the median raw next-step MSE across prediction horizon on each other's fixed induced states. Each panel holds the evaluated model fixed and varies only the induced-state source, so this is a frozen-state diagnostic rather than a deployable rollout metric. Unlike the normalized cross-state comparison, the raw MSE makes absolute state difficulty visible: SSP-induced states often yield lower next-step error not only for SSP itself, but also for TF and SS. This suggests that correction training helps not only by improving the local predictor on a fixed induced-state dataset, but also by changing the induced-state regime encountered during recursive deployment. Corresponding results for MG and Weather are given in Appendix~\ref{sec:appendix_cross_state_raw}.
}
    \label{fig:ss_bridge_cross_state}
\end{figure*}

\section{Discussion and Conclusion}
\label{sec:disc}

Our central claim is not merely that one-step optimal recursive forecasting need not coincide with multi-step optimal forecasting. It is that under partial observability or state truncation, the one-step Bayes objective need not identify the deployed recursive predictor itself. One-step supervision constrains behavior on observed contexts, while recursive deployment queries self-generated induced states whose local interpretation need not be determined by numeric state alone. In this sense, exposure bias is not only train--test mismatch, but a problem of epistemic underidentification in the deployed recursive task.

The empirical results support this claim in three steps. First, rollout enters a regime increasingly distinct from the observed training contexts, making the identification problem practically relevant. Second, frozen induced-state evaluation shows that recursive deployment creates a distinct local corrective task: post-hoc relearning is heterogeneous across datasets and is not a uniform improvement over the original teacher-forced predictor. Third, closed-loop correction improves rollout partly by changing the induced states encountered during deployment, not only by fitting a better local predictor on a fixed induced-state dataset. Provenance-aware correction is most suggestive on ETTh1, but remains mixed overall, consistent with the theory that provenance helps only when it carries target-relevant information recoverable under the chosen encoding. The theory also separates two mechanisms that are often conflated in discussions
of exposure bias: disagreement at induced states whose behavior is
underidentified by one-step training, and amplification of earlier
discrepancies through the closed-loop dynamics. Our experiments do not isolate
these mechanisms cleanly, but the contrast between frozen-state and closed-loop
evaluation is consistent with both being relevant in practice.

Viewed through this lens, provenance matters not merely as an auxiliary input, but because it can resolve aliasing in the induced-state task: the same numeric state may correspond to different corrective targets depending on how it was produced. The toy example demonstrates this mechanism exactly. The real-data results suggest that such information can matter in practice, though the present binary encoding recovers it only weakly and inconsistently. For method design, the results suggest that progress may come not only from improving local predictors, but from exploiting information about how rollout states were formed. This perspective shifts the narrative from viewing exposure bias only as covariate shift toward treating recursive deployment as prediction under self-induced epistemic uncertainty.

\paragraph{Future work.} Our provenance signal is deliberately minimal: a binary indicator of which inputs were model-generated. The mixed SS--SSP gains therefore should not be interpreted as a verdict on provenance as such, but on this particular encoding in a finite-data, restricted-model setting. Richer induced-state summaries may be more effective, including rollout depth, uncertainty estimates, latent-state summaries, or structured traces of state construction.

% \newpage
% --- BIBLIOGRAPHY ---
\bibliographystyle{plainnat}
\bibliography{references}

\newpage

\appendix

\section{Expanded Theory Section}
\label{app:ext_theory}
We give brief proofs of the claims in Section~2. Throughout, let
\[
  M := \operatorname{supp}\!\big(\mathcal{L}(X_t)\big)
\]
denote the support of the observed-context distribution. For a measurable
predictor $g$, define the closed-loop transition
\[
  T_g(x_1,\ldots,x_{\hat p})
  :=
  \bigl(g(x_1,\ldots,x_{\hat p}),x_1,\ldots,x_{\hat p-1}\bigr),
\]
and the two-step recursive forecast
\[
  \Phi_g(x) := g\!\bigl(T_g(x)\bigr).
\]

\subsection{One-step Bayes risk under state truncation}

\begin{proof}
For any measurable predictor $g$,
\[
  R_1(g)=\mathbb{E}\big[(Y_{t+1}-g(X_t))^2\big].
\]
Add and subtract $\mathbb{E}[Y_{t+1}\mid X_t]$:
\begin{align*}
  Y_{t+1}-g(X_t)
  &=
  \Big(Y_{t+1}-\mathbb{E}[Y_{t+1}\mid X_t]\Big) \\
  &\quad+
  \Big(\mathbb{E}[Y_{t+1}\mid X_t]-g(X_t)\Big).
\end{align*}
After squaring and taking expectations, the cross term vanishes because
\[
  \mathbb{E}\!\left[
    Y_{t+1}-\mathbb{E}[Y_{t+1}\mid X_t]
    \,\middle|\,
    X_t
  \right]
  = 0.
\]
Hence
\begin{align*}
  R_1(g)
  &=
  \mathbb{E}\!\left[\Var(Y_{t+1}\mid X_t)\right] \\
  &\quad+
  \mathbb{E}\!\left[
    \big(\mathbb{E}[Y_{t+1}\mid X_t]-g(X_t)\big)^2
  \right].
\end{align*}
Therefore the Bayes predictor
\[
  g^\star(x)=\mathbb{E}[Y_{t+1}\mid X_t=x]
\]
minimizes $R_1(g)$ and satisfies
\[
  R_1(g^\star)
  =
  \mathbb{E}\!\left[\Var(Y_{t+1}\mid X_t)\right].
\]
This is strictly positive exactly when $Y_{t+1}$ is not almost surely a
measurable function of $X_t$.
\end{proof}

\subsection{Proof of Theorem 1}

% The theorem is an ``in general'' statement, so it is enough to exhibit one
% sufficient case where one-step optimality fails to determine recursive rollout.
Assume there exists $x\in M$ such that
\[
  z := T_{g^\star}(x) \notin M.
\]

\begin{proof}
Since $z\notin \operatorname{supp}(\mathcal{L}(X_t))$, there exists an open
set $U\ni z$ such that $\mathbb{P}(X_t\in U)=0$. In particular, $U\cap M=\varnothing$.

Define
\[
  g_1 := g^\star,
  \qquad
  g_2(u)
  :=
  \begin{cases}
    g^\star(u), & u\notin U,\\
    c, & u\in U,
  \end{cases}
\]
for any constant $c\neq g^\star(z)$. Then $g_1$ and $g_2$ agree on $M$, hence
\[
  g_1(X_t)=g_2(X_t)=g^\star(X_t)
  \qquad\text{a.s.}
\]
So they achieve the same one-step risk:
\[
  R_1(g_1)=R_1(g_2)=R_1(g^\star).
\]

Now evaluate them recursively from the initial state $x$. Since
$g_1(x)=g_2(x)=g^\star(x)$,
\[
  T_{g_1}(x)=T_{g_2}(x)=T_{g^\star}(x)=z.
\]
Therefore
\[
  \Phi_{g_1}(x)=g_1(z)=g^\star(z),
  \qquad
  \Phi_{g_2}(x)=g_2(z)=c.
\]
Because $c\neq g^\star(z)$, we have
\[
  \Phi_{g_1}(x)\neq \Phi_{g_2}(x).
\]
Thus one-step Bayes optimality does not determine recursive rollout outside the
support of observed contexts.
\end{proof}

\subsection{Additional two-step expansion: off-support mismatch and amplification}
\label{app:two-step-divergence}

The following first-order expansion is not needed for the main theorem
statements, but it helps interpret more precisely how one-step agreement can
still yield different recursive forecasts. In particular, it gives a local
mechanistic refinement of Theorem~\ref{thm:one-step-optimality-recursive-rollout}
by separating disagreement at the induced state from amplification of an earlier
first-step discrepancy.

\begin{proposition}[Two-step divergence separates two failure modes]
\label{prop:two-step-divergence-failure-modes}
Let $\Phi_g(x)$ denote the two-step recursive forecast induced by predictor
$g$. For two predictors $g_1,g_2$, define
\[
  \delta(x):=g_1(x)-g_2(x).
\]
Assume $g_1$ and $g_2$ are $C^1$ in a neighborhood of $x$ and
$T_{g_2}(x)$. Then
\begin{align*}
  \Phi_{g_1}(x)-\Phi_{g_2}(x)
  &=
  \bigl(g_1-g_2\bigr)\!\bigl(T_{g_2}(x)\bigr) \\
  &\quad+
  \partial_1 g_1\!\bigl(T_{g_2}(x)\bigr)\,\delta(x) \\
  &\quad+
  o\!\bigl(|\delta(x)|\bigr).
\end{align*}
\end{proposition}

\begin{proof}
By definition,
\[
  \Phi_{g_1}(x)-\Phi_{g_2}(x)
  =
  g_1\!\bigl(T_{g_1}(x)\bigr)-g_2\!\bigl(T_{g_2}(x)\bigr).
\]
Add and subtract $g_1\!\bigl(T_{g_2}(x)\bigr)$:
\begin{align*}
  \Phi_{g_1}(x)-\Phi_{g_2}(x)
  &=
  \Big(
    g_1\!\bigl(T_{g_1}(x)\bigr)-g_1\!\bigl(T_{g_2}(x)\bigr)
  \Big) \\
  &\quad+
  \bigl(g_1-g_2\bigr)\!\bigl(T_{g_2}(x)\bigr).
\end{align*}
Also,
\[
  T_{g_1}(x)-T_{g_2}(x)
  =
  \bigl(\delta(x),0,\ldots,0\bigr).
\]
A first-order Taylor expansion of $g_1$ at $T_{g_2}(x)$ gives
\[
  g_1\!\bigl(T_{g_1}(x)\bigr)
  =
  g_1\!\bigl(T_{g_2}(x)\bigr)
  +
  \partial_1 g_1\!\bigl(T_{g_2}(x)\bigr)\,\delta(x)
  +
  o\!\bigl(|\delta(x)|\bigr).
\]
Substituting yields the claim.
\end{proof}

\paragraph{Interpretation.}
Proposition~\ref{prop:two-step-divergence-failure-modes} separates two local
mechanisms by which recursive forecasts can diverge. The term
\[
  \bigl(g_1-g_2\bigr)\!\bigl(T_{g_2}(x)\bigr)
\]
captures disagreement at the induced state queried by rollout. This includes
off-support states whose local behavior is not identified by one-step
supervision, and is the epistemic component emphasized in the main text. By
contrast, the term
\[
  \partial_1 g_1\!\bigl(T_{g_2}(x)\bigr)\,\delta(x)
\]
captures Jacobian-mediated amplification of an already-present first-step
discrepancy through the closed-loop dynamics. The proposition should therefore
be read as a local complement to the main information argument: recursive error
can arise both from underidentified behavior on self-generated states and from
dynamical sensitivity to earlier local error.

\subsection{Proof of Theorem 2}

\begin{proof}
For any measurable representation $W$, the Bayes risk under squared loss is
\[
  R^\star(W)
  :=
  \inf_f \mathbb{E}\big[(Y-f(W))^2\big]
  =
  \mathbb{E}\!\left[\Var(Y\mid W)\right].
\]
Now $Z_h=\psi_h(X_t)$ and $P_h=\pi_h(X_t)$ are deterministic measurable
functions of $X_t$, so
\[
  \sigma(Z_h)\subseteq \sigma(Z_h,P_h)\subseteq \sigma(X_t).
\]
Since conditional variance decreases as the conditioning $\sigma$-field becomes
finer,
\[
  \mathbb{E}\!\left[\Var(Y\mid X_t)\right]
  \le
  \mathbb{E}\!\left[\Var(Y\mid Z_h,P_h)\right]
  \le
  \mathbb{E}\!\left[\Var(Y\mid Z_h)\right].
\]
Equivalently,
\[
  R^\star(X_t)
  \le
  R^\star_{\mathrm{prov}}(Z_h,P_h)
  \le
  R^\star(Z_h).
\]

To characterize the right inequality, apply the law of total variance
conditional on $Z_h$:
\begin{align*}
  \Var(Y\mid Z_h)
  &=
  \mathbb{E}\!\left[
    \Var(Y\mid Z_h,P_h)\mid Z_h
  \right] \\
  &\quad+
  \Var\!\left(
    \mathbb{E}[Y\mid Z_h,P_h]\mid Z_h
  \right).
\end{align*}
Taking expectations gives
\begin{align*}
  R^\star(Z_h)-R^\star_{\mathrm{prov}}(Z_h,P_h)
  &=
  \mathbb{E}\!\left[
    \Var\!\left(
      \mathbb{E}[Y\mid Z_h,P_h]\mid Z_h
    \right)
  \right] \\
  &\ge 0.
\end{align*}
Hence the right inequality is strict exactly when
\[
  \mathbb{E}[Y\mid Z_h,P_h]\neq \mathbb{E}[Y\mid Z_h]
\]
on a set of positive probability.

Similarly, because $(Z_h,P_h)$ is coarser than $X_t$,
\begin{align*}
  \Var(Y\mid Z_h,P_h)
  &=
  \mathbb{E}\!\left[
    \Var(Y\mid X_t)\mid Z_h,P_h
  \right] \\
  &\quad+
  \Var\!\left(
    \mathbb{E}[Y\mid X_t]\mid Z_h,P_h
  \right).
\end{align*}
Taking expectations yields
\begin{align*}
  R^\star_{\mathrm{prov}}(Z_h,P_h)-R^\star(X_t)
  &=
  \mathbb{E}\!\left[
    \Var\!\left(
      \mathbb{E}[Y\mid X_t]\mid Z_h,P_h
    \right)
  \right] \\
  &\ge 0.
\end{align*}
So the left inequality is strict whenever $(Z_h,P_h)$ discards
target-relevant information from $X_t$.

Finally,
\begin{align*}
  R^\star(Z_h)-R^\star(X_t)
  &=
  \Big(
    R^\star(Z_h)-R^\star_{\mathrm{prov}}(Z_h,P_h)
  \Big) \\
  &\quad+
  \Big(
    R^\star_{\mathrm{prov}}(Z_h,P_h)-R^\star(X_t)
  \Big),
\end{align*}
which gives the Bayes-risk decomposition underlying the information sandwich.
\end{proof}

\subsection{Derivation of Eq.~\texorpdfstring{\eqref{eq:decomp}}{(1)}}

To keep the display short, write
\[
  R_{h,Z}^\star := R_h^\star(Z_h),
  \qquad
  R_{h,ZP}^\star := R_h^{\mathrm{prov},\star}(Z_h,P_h).
\]
Then Eq.~\eqref{eq:decomp} is obtained by adding and subtracting intermediate terms,
\begin{align}
  R_h(g_{\mathrm{TF}})-R_{h,ZP}^\star
  &=
  \Big(
    R_h(g_{\mathrm{TF}})
    - \inf_{q\in\mathcal{Q}} R_h(q)
  \Big) \notag\\
  &\quad+
  \Big(
    \inf_{q\in\mathcal{Q}} R_h(q)
    - R_{h,Z}^\star
  \Big) \notag\\
  &\quad+
  \Big(
    R_{h,Z}^\star
    - R_{h,ZP}^\star
  \Big).
\end{align}
Each term is nonnegative:
\begin{enumerate}
  \item the best element of $\mathcal{Q}$ cannot be worse than the specific
  map $g_{\mathrm{TF}}$;
  \item Bayes risk is minimal over all measurable functions of $Z_h$;
  \item conditioning on $(Z_h,P_h)$ cannot increase Bayes risk under squared
  loss.
\end{enumerate}
Thus the three terms in Eq.~\eqref{eq:decomp} correspond exactly to
teacher-forcing/rollout mismatch, representation--class approximation, and the
provenance information gap.

% \newpage
\section{Toy example: binary provenance disambiguates an induced state}
\label{app:toy-provenance}

This toy makes the provenance term in
Theorem~\ref{thm:information-sandwich-induced-states-provenance} concrete.  We
use a deterministic lag-3 delay system
\[
  y_{t+1}=y_{t-2},
\]
but the predictor only receives a lag-$2$ window
\[
  X_t=(y_t,y_{t-1}).
\]
The point of the construction is to exhibit a case in which the \emph{same
numeric state} occurs both on the observed support and as a rollout-induced
state, with different correct next-step targets.

We consider three local histories,
\[
  H_1=(1,1,1), \qquad H_2=(1,1,-1), \qquad H_3=(0,1,-1),
\]
where each tuple denotes $(y_t,y_{t-1},y_{t-2})$. Their observed lag-$2$
windows and one-step targets are
\[
  X(H_1)=X(H_2)=(1,1), \qquad X(H_3)=(0,1),
\]
and
\[
  y_{t+1}(H_1)=1,\qquad y_{t+1}(H_2)=-1,\qquad y_{t+1}(H_3)=-1.
\]
Hence the Bayes-optimal one-step predictor on the observed support satisfies
\[
  g^\star(1,1)=0, \qquad g^\star(0,1)=-1.
\]
So the observed state $(1,1)$ is Bayes-ambiguous, while the observed state
$(0,1)$ has target $-1$.

Now roll out one Bayes step from the ambiguous observed state $(1,1)$. Since
$g^\star(1,1)=0$, the lag-$2$ update yields the induced state
\[
  Z_t=(0,1).
\]
Crucially, this is \emph{the same numeric state} as the observed state produced
by $H_3$. But its correct local next-step target is different. For both $H_1$
and $H_2$, after one rollout step the correct next target is
\[
  y_{t+2}=y_{t-1}=1.
\]
Thus the numeric state $(0,1)$ carries two different local targets:
\[
  (0,1)\ \text{observed} \;\mapsto\; -1,
  \qquad
  (0,1)\ \text{induced} \;\mapsto\; +1.
\]

This is exactly the clash shown in Figure~\ref{fig:toy_prov}. A predictor that
uses only the numeric state must conflate these two cases. In particular, a
correction dataset formed from induced states contains the pair
\[
  ((0,1),+1),
\]
while the teacher-forced observed dataset contains
\[
  ((0,1),-1).
\]
Without provenance, these two supervision signals are indistinguishable.

A binary provenance variable resolves the ambiguity. Let
\[
  P_t\in\{0,1\}^2
\]
encode whether each entry of the lag-$2$ input is observed ($0$) or
model-generated ($1$). Then
\[
  P_t=(0,0)
\]
for the observed state $(0,1)$ from $H_3$, whereas after one rollout step from
$(1,1)$ we have
\[
  P_t=(1,0),
\]
because the first coordinate is the Bayes prediction $0$ and the second remains
observed. Therefore
\begin{align}
\E[Y \mid Z_t=(0,1), P_t=(0,0)] &= -1, \\
\E[Y \mid Z_t=(0,1), P_t=(1,0)] &= +1.
\end{align}
By contrast, $\E[Y \mid Z_t=(0,1)]$ necessarily averages over these cases. In this construction, provenance is
informative beyond the numeric induced state, so the right inequality in
Theorem~\ref{thm:information-sandwich-induced-states-provenance} is strict.

The toy should be read as an existence proof, not as a claim that every
provenance encoding is sufficient. Its point is that even a simple binary
observed/generated indicator can become target-relevant when the same numeric
state appears in both observed and induced regimes with different meanings. Crucially, this clash arises even though the induced state is generated by the
Bayes-optimal one-step predictor on the observed task. Thus the ambiguity cannot
be attributed to estimation error on the training distribution; it is a property
of the induced-state prediction problem itself.

% \newpage
\section{Experimental Setup}
\label{app:experimental_setup}
\subsection{Datasets}

Experiments are conducted on three time-series datasets.

\textbf{Mackey--Glass (MG).} A synthetic chaotic time series generated from the Mackey--Glass delay differential equation with parameters $a = 0.1$, $\gamma = 0.2$, delay $\tau = 17$, and process noise standard deviation $\sigma_s = 0.2$ (measurement noise $\sigma_e = 0$). The system is initialized at a constant value of $1.0$ with a burn-in of $100$ steps discarded before recording. This dataset allows controlled study of the induced-state regime under known chaotic dynamics.

\textbf{Real-world benchmarks.} Two datasets from the BasicTS benchmark suite are used: ETTh1 (electricity transformer temperature, hourly) and Weather (meteorological measurements). For these real-world datasets, a single univariate series is extracted (channel index $0$) and standardized to zero mean and unit variance prior to windowing.

\subsection{Data Splits and Windowing}

Each series is split sequentially into a training segment of $4{,}000$ steps, a validation segment of $2{,}000$ steps, and a test segment of $2{,}000$ steps (total $8{,}000$ steps). Sliding-window input--output pairs $(\mathbf{x}_t, \mathbf{y}_t)$ are formed with a context window of size $W = 20$ and a multi-step horizon of $H = 20$, yielding target vectors $\mathbf{y}_t = (y_{t+1}, \ldots, y_{t+H})$.

\begin{figure*}[t]
    \centering

    \begin{subfigure}[t]{0.32\textwidth}
        \centering
        \includegraphics[width=\linewidth]{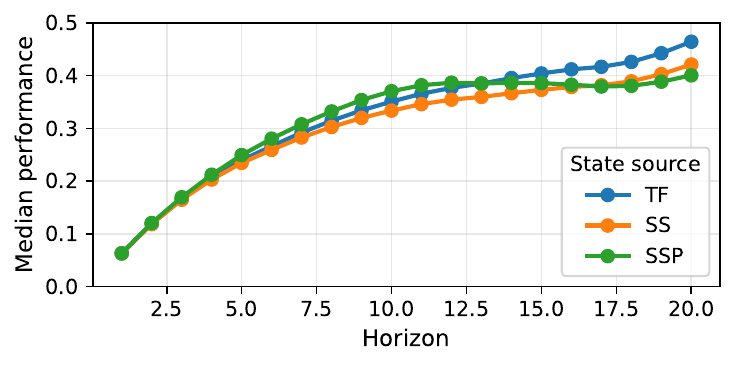}
        \caption{TF model}
    \end{subfigure}\hfill
    \begin{subfigure}[t]{0.32\textwidth}
        \centering
        \includegraphics[width=\linewidth]{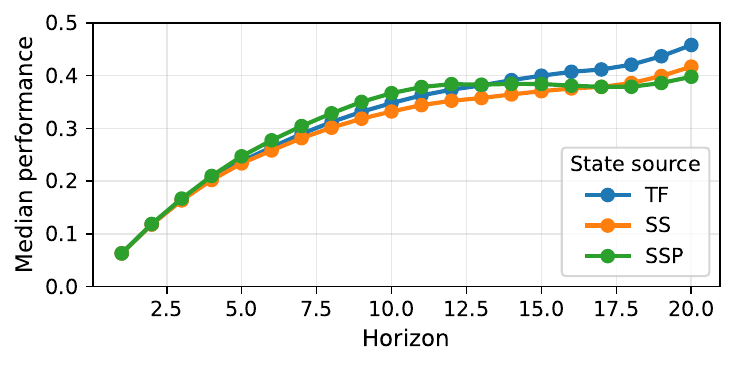}
        \caption{SS model}
    \end{subfigure}\hfill
    \begin{subfigure}[t]{0.32\textwidth}
        \centering
        \includegraphics[width=\linewidth]{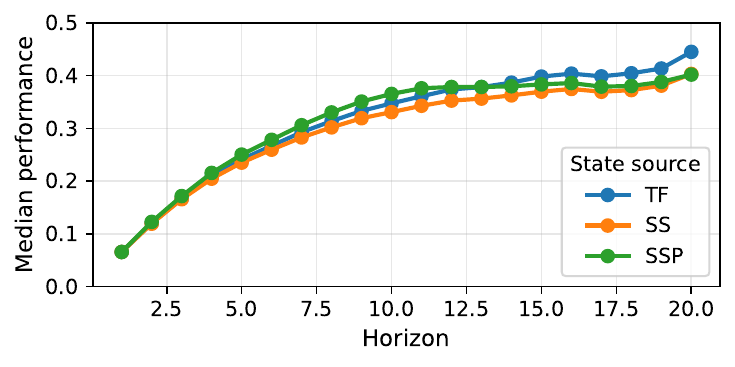}
        \caption{SSP model}
    \end{subfigure}

\caption{
\textbf{MG raw cross-state next-step MSE.}
Median raw next-step MSE across rollout horizon for TF, SS, and SSP evaluated on fixed induced states generated by TF, SS, or SSP. Each panel fixes the evaluated model and varies only the state source.
}
    \label{fig:cross_state_raw_mg}
\end{figure*}

\subsection{Models and Training}

All models were implemented in PyTorch. We evaluated two recursive one-step forecasters: a GRU and an MLP. The GRU used hidden size $128$, $1$ layer, and dropout $0.0$. The MLP used a single hidden layer of width $128$ with \texttt{tanh} activations and dropout $0.0$. In all cases, models output a scalar next-step prediction and were rolled out recursively for multi-step forecasting.

For each dataset, architecture, and seed, we trained three variants: teacher forcing (TF), scheduled sampling without provenance (SS), and scheduled sampling with provenance (SSP). SS and SSP used a linear teacher-forcing schedule from $1.0$ to $0.2$ over training. In SSP, the model additionally received a binary provenance indicator marking which entries in the input window were model-generated rather than observed.

Training used SGD with learning rate $10^{-3}$$, $batch size $128$, no weight decay, and $500$ epochs. Model selection was based on validation rollout MSE, and the checkpoint with the best validation rollout MSE was used for test-time evaluation.

\subsection{Evaluation Protocol}

\begin{figure*}[t]
    \centering

    \begin{subfigure}[t]{0.32\textwidth}
        \centering
        \includegraphics[width=\linewidth]{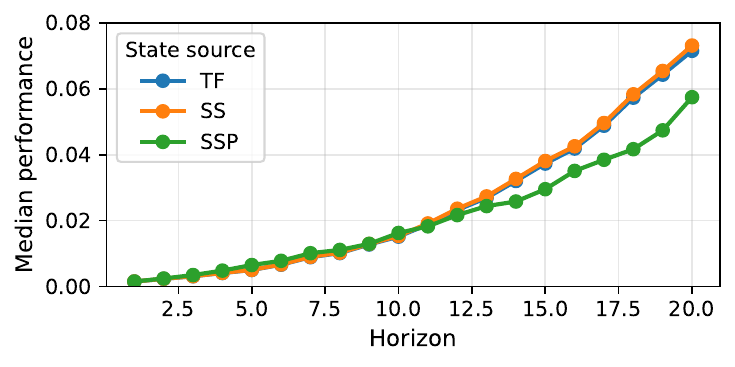}
        \caption{TF model}
    \end{subfigure}\hfill
    \begin{subfigure}[t]{0.32\textwidth}
        \centering
        \includegraphics[width=\linewidth]{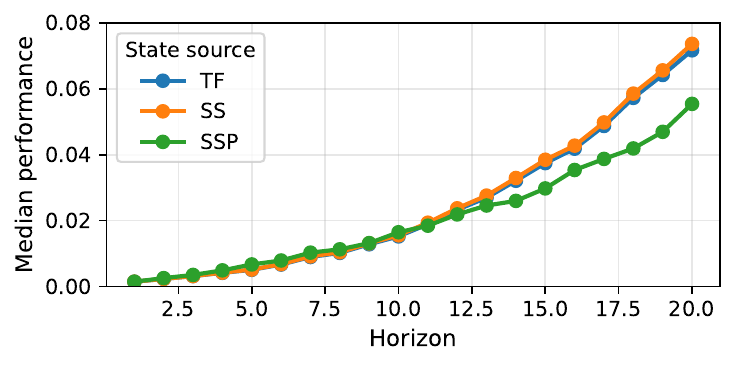}
        \caption{SS model}
    \end{subfigure}\hfill
    \begin{subfigure}[t]{0.32\textwidth}
        \centering
        \includegraphics[width=\linewidth]{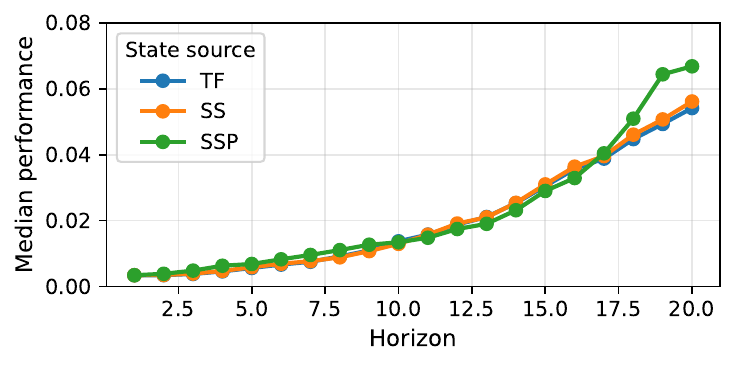}
        \caption{SSP model}
    \end{subfigure}

\caption{
\textbf{Weather raw cross-state next-step MSE.}
Median raw next-step MSE across rollout horizon for TF, SS, and SSP evaluated on fixed induced states generated by TF, SS, or SSP. Each panel fixes the evaluated model and varies only the state source.
}
    \label{fig:cross_state_raw_weather}
\end{figure*}

We report recursive rollout performance on train, validation, and test splits, with primary comparisons made on the test set. For each trained model we saved full rollout predictions and induced-state objects over the forecasting horizon. The appendix figures and tables were then computed from these saved rollouts.

For the rollout-regime diagnostic, we trained a linear logistic-regression probe to distinguish true observed states from induced rollout states at each horizon, using $256$ sampled states per depth, $10$ repeats, and an $80/20$ train/test split. For the frozen induced-state analyzes, probe regressors matched the main forecaster family (GRU probes for GRU experiments, MLP probes for MLP experiments), were trained on the training induced states, and used the same main optimization settings except that provenance-aware probes additionally received the provenance variables with scaling factor $\lambda_P = 1.0$.

\subsection{Runs}

The final experiment grid comprised $3$ datasets (MG, ETTh1, Weather), $2$ model classes (GRU, MLP), and $5$ seeds for randomisation, 0-4, for a total of $30$ runs.

% \FloatBarrier

% \newpage

\section{Cross-state experiments: raw next-step MSE by induced-state source}
\label{sec:appendix_cross_state_raw}

This appendix reports raw cross-state next-step MSE results for the datasets not shown in the main text. As in the ETTh1 figure, we evaluate each predictor on fixed induced states generated by TF, SS, or SSP, and measure next-step MSE across rollout depth. Each panel fixes the evaluated predictor and varies only the induced-state source, so these plots isolate frozen-state behavior rather than deployable rollout.

The raw MSE makes absolute state difficulty visible. Across datasets, SSP-induced states often yield lower next-step error not only for SSP itself, but also for TF and SS, particularly at deeper horizons. On MG, this effect emerges later in rollout, with SS states initially easier but SSP states becoming easier at greater depth. On Weather, SSP states are generally easier throughout and increasingly so with depth. 

These results support the main-text interpretation: correction helps not only by improving local prediction on a fixed induced-state dataset, but also by changing the induced-state regime encountered during recursive deployment.

% \FloatBarrier
% \newpage
% \clearpage
\section{Robustness check with GRU}
\label{app:gru_robustness}

This appendix repeats the main empirical analyzes using a GRU forecaster instead of the MLP used in the main text. The qualitative conclusions are unchanged: recursive rollout still enters a distinct induced-state regime, frozen induced-state relearning remains heterogeneous across datasets, and the mismatch between frozen-state local evaluation and deployable rollout again suggests that correction helps partly by changing the states visited during rollout.

\paragraph{Rollout regime check.}
Figure~\ref{fig:gru_regime_check} is the GRU analogue of Figure~\ref{fig:regime_check} in the main text. As in the MLP case, probe accuracy generally increases with rollout depth, with the strongest separation on ETTh1, moderate separation on MG, and weaker separation on Weather. This again supports the view that recursive deployment progressively leaves the observed-state regime.

\begin{figure}[t]
    \centering
    \includegraphics[width=\linewidth]{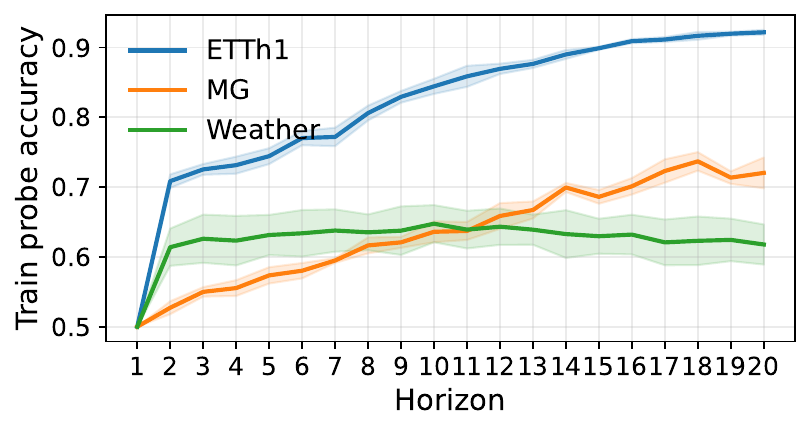}
    \caption{
    \textbf{GRU analogue of Figure~\ref{fig:regime_check}.}
    Same rollout-regime diagnostic as in the main text, but for the GRU forecaster. The qualitative pattern is the same: observed contexts and teacher-forced induced states become increasingly distinguishable with rollout depth, most strongly on ETTh1, moderately on MG, and weakly on Weather.
    }
    \label{fig:gru_regime_check}
\end{figure}

\paragraph{Frozen induced-state evaluation.}
Figure~\ref{fig:gru_induced_decomposition} is the GRU analogue of Figure~\ref{fig:induced_decomposition}. The same overall pattern appears here: relearning on fixed induced states is not uniformly beneficial, gains remain dataset-dependent, and the provenance-aware probe stays close to the $Z$-only probe. As in the main text, this suggests that frozen induced states define a distinct prediction problem, but that the present provenance encoding adds only limited extra signal in this setting.

\begin{figure*}[t]
    \centering

    \begin{subfigure}[t]{0.32\textwidth}
        \centering
        \includegraphics[width=\linewidth]{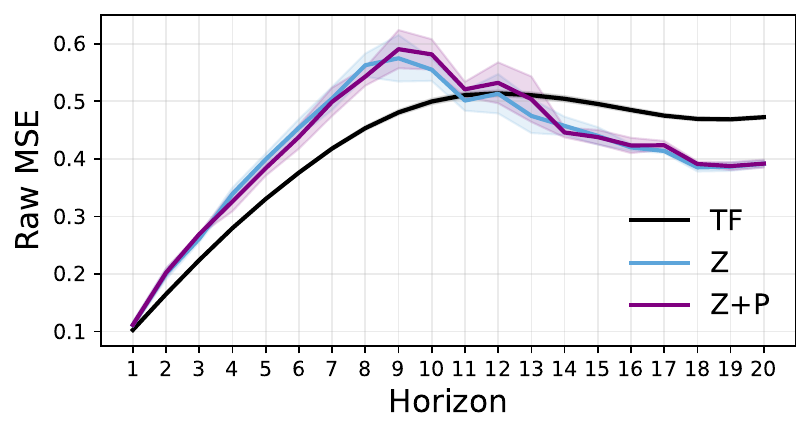}
        \caption{MG}
    \end{subfigure}\hfill
    \begin{subfigure}[t]{0.32\textwidth}
        \centering
        \includegraphics[width=\linewidth]{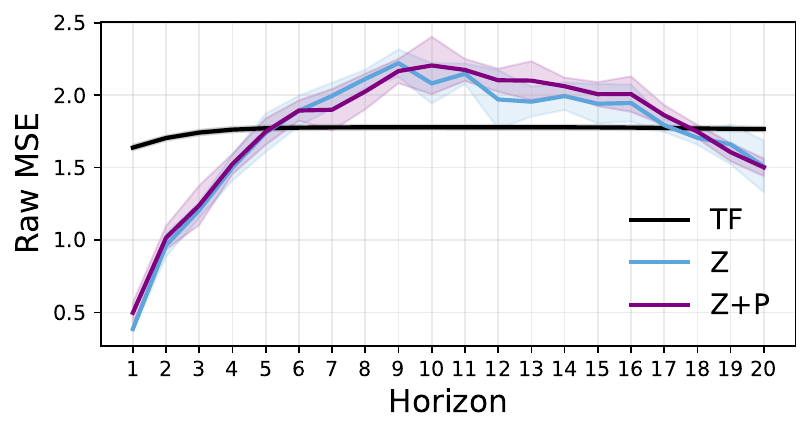}
        \caption{ETTh1}
    \end{subfigure}\hfill
    \begin{subfigure}[t]{0.32\textwidth}
        \centering
        \includegraphics[width=\linewidth]{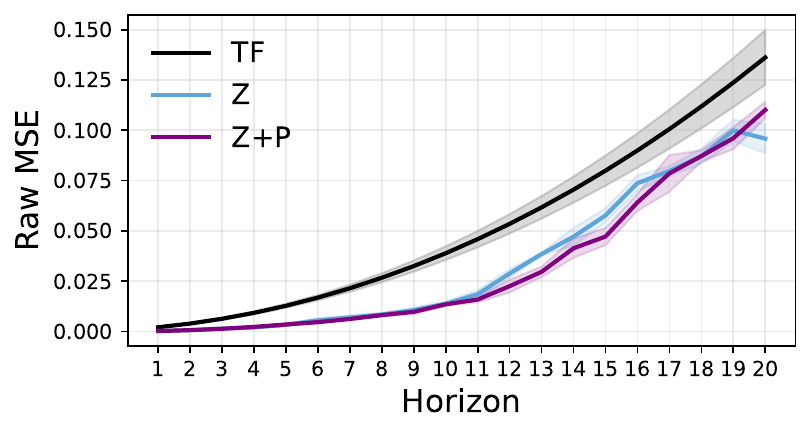}
        \caption{Weather}
    \end{subfigure}

    \caption{
    \textbf{GRU analogue of Figure~\ref{fig:induced_decomposition}.}
    Same frozen induced-state evaluation as in the main text, but for the GRU forecaster. The qualitative conclusions are unchanged: relearning on fixed induced states is heterogeneous across datasets, and provenance-aware probes remain close to $Z$-only probes.
    }
    \label{fig:gru_induced_decomposition}
\end{figure*}

\paragraph{Cross-state local next-step evaluation.}
Figure~\ref{fig:gru_ss_bridge_cross_state} is the GRU analogue of Figure~\ref{fig:ss_bridge_cross_state}. The same bridge pattern appears: frozen-state local comparisons do not fully track deployable rollout gains. This again supports the interpretation that correction training helps not only by improving the local predictor on a fixed induced-state dataset, but also by changing which induced states are encountered during recursive deployment.

\begin{figure*}[t]
    \centering

    \begin{subfigure}[t]{0.32\textwidth}
        \centering
        \includegraphics[width=\linewidth]{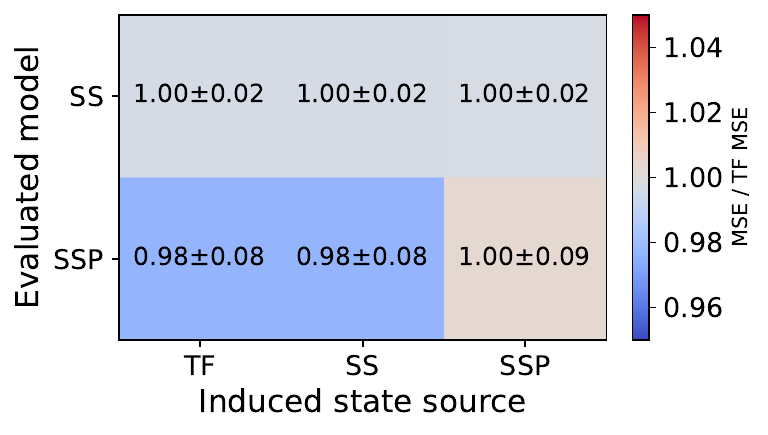}
        \caption{Early horizon}
    \end{subfigure}\hfill
    \begin{subfigure}[t]{0.32\textwidth}
        \centering
        \includegraphics[width=\linewidth]{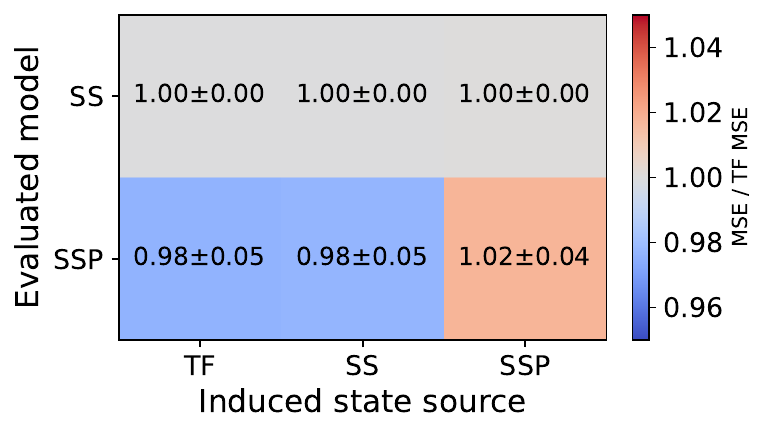}
        \caption{Mid horizon}
    \end{subfigure}\hfill
    \begin{subfigure}[t]{0.32\textwidth}
        \centering
        \includegraphics[width=\linewidth]{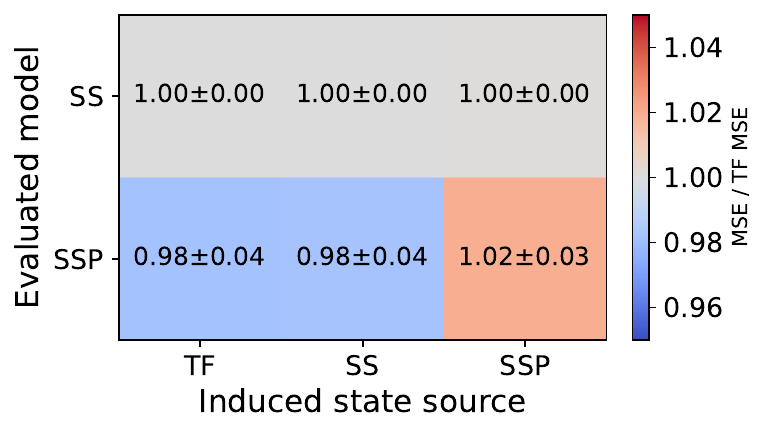}
        \caption{Late horizon}
    \end{subfigure}

    \caption{
    \textbf{GRU analogue of Figure~\ref{fig:ss_bridge_cross_state}.}
    Same cross-state local next-step diagnostic as in the main text, but for the GRU forecaster. The qualitative pattern is again similar: frozen-state local improvements do not fully explain deployable rollout gains, consistent with correction helping partly by changing the induced-state regime itself.
    }
    \label{fig:gru_ss_bridge_cross_state}
\end{figure*}

\end{document}